\documentclass[11pt,a4paper]{article}
\usepackage[hyperref]{acl2021}
\usepackage{times}
\usepackage{latexsym}

\usepackage{microtype}
\usepackage{makecell}
\usepackage{multirow}
\usepackage{pifont}
\usepackage{comment}
\newcommand{\cmark}{\ding{51}}%
\newcommand{\xmark}{\ding{55}}%
\usepackage{bm}
\usepackage{pifont}
\usepackage{amssymb}
\usepackage{graphicx}
\usepackage{subcaption}
\usepackage{amsthm}
\newtheorem{proposition}{Proposition}
\usepackage{CJK}
\usepackage{wrapfig}
\usepackage{array}
\usepackage{amsmath}

\aclfinalcopy % Uncomment this line for the final submission
\title{PairRE: Knowledge Graph Embeddings via Paired Relation Vectors}

\author{Linlin Chao, Jianshan He, Taifeng Wang, Wei Chu \\
    AntGroup \\
  \texttt{\{chulin.cll,yebai.hjs\}@antgroup.com} \\ 
  \texttt{\{taifeng.wang,wei.chu\}@alibaba-inc.com} 
 }

\date{}

\begin{document}
\maketitle
\begin{abstract}
Distance based knowledge graph embedding methods show promising results on link prediction task, on which two topics have been widely studied: 
one is the ability to handle complex relations, such as N-to-1, 1-to-N and N-to-N, 
the other is to encode various relation patterns, such as symmetry/antisymmetry.
However, the existing methods fail to solve these two problems at the same time, which leads to unsatisfactory results.
To mitigate this problem, we propose PairRE, a model with paired vectors for each relation representation. 
The paired vectors enable an adaptive adjustment of the margin in loss function to fit for complex relations. Besides, PairRE is capable of encoding three important relation patterns, symmetry/antisymmetry, inverse and composition.
Given simple constraints on relation representations, PairRE can encode subrelation further.
Experiments on link prediction benchmarks demonstrate the proposed key capabilities of PairRE. Moreover, We set a
new state-of-the-art on two knowledge graph datasets of the challenging Open Graph Benchmark.
\end{abstract}

\section{Introduction}

Knowledge graphs store huge amounts of structured data in the form of triples, with projects such as WordNet \cite{miller1995wordnet}, Freebase  \cite{bollacker2008freebase}, YAGO \cite{suchanek2007yago} and DBpedia \cite{lehmann2015dbpedia}.
They have gained widespread attraction from their successful use in tasks such as question answering \cite{bordes2014open}, semantic parsing \cite{berant2013semantic}, and named entity disambiguation \cite{zheng2012entity} and so on.

Since most knowledge graphs suffer from incompleteness, predicting missing links between entities has been a fundamental problem.
This problem is named as link prediction or knowledge graph completion.
Knowledge graph embedding methods, which embed all entities and relations into a low dimensional space, have been proposed for this problem.

Distance based embedding methods from TransE \cite{bordes2013translating} to the recent state-of-the-art RotatE \cite{sun2019rotate} have shown substantial improvements on knowledge graph completion task. Two major problems have been widely studied.
{The first one refers to handling of 1-to-N, N-to-1, and N-to-N complex relations \cite{bordes2013translating, lin2015learning}}.
In case of the 1-to-N relations, given triples like ($StevenSpielberg$, $DirectorOf$, $?$), distance based models should make all the corresponding entities about film name like $Jaws$ and $JurassicPark$ have closer distance to entity $StevenSpielberg$ after transformation via relation $DirectorOf$. 
The difficulty is that all these entities should have different representations. Same issue happens in cases of N-to-N and N-to-1 relations. 
{The latter is learning and inferring relation patterns according to observed triples, as the success of knowledge graph completion heavily relies on this ability \cite{bordes2013translating, sun2019rotate}}.
There are various types of relation patterns: symmetry (e.g., $IsSimilarTo$), antisymmetry (e.g., $FatherOf$), inverse (e.g., $PeopleBornHere$ and $PlaceOfBirth$), composition (e.g., my mother’s father is my grandpa) and so on. 

Previous methods solve these two problems separately.
TransH \cite{wang2014knowledge}, TransR \cite{lin2015learning}, TransD \cite{ji2015knowledge} all focus on ways to solve complex relations.
However, these methods can only encode symmetry/antisymmetry relations.
The recent state-of-the-art RotatE shows promising results to encode symmetry/antisymmetry, inverse and composition relations.
However, complex relations remain challenging to predict.

Here we present PairRE, an embedding method that is capable of encoding complex relations and multiple relation patterns simultaneously. The proposed model uses two vectors for relation representation. These vectors project the corresponding head and tail entities to Euclidean space, where the distance between the projected vectors is minimized.
This provides three important benefits:
\begin{itemize}
\item[$\bullet$]
The paired relation representations enable an adaptive adjustment of the margin in loss function to fit for different complex relations;
\end{itemize}
\begin{itemize}
\item[$\bullet$]
Semantic connection among relation vectors can be well captured, which enables the model to encode three important relation patterns, symmetry/antisymmetry, inverse and composition;
\end{itemize}
\begin{itemize}
\item[$\bullet$]
Adding simple constraints on relation representations, PairRE can encode subrelation further.
\end{itemize}
Besides, PairRE is a highly efficient model, which contributes to large scale datasets.

We evaluate PairRE on six standard knowledge graph benchmarks. The experiment results show PairRE can achieve either state-of-the-art or highly competitive performance. Further analysis also proves that PairRE can better handle complex relations and encode symmetry/antisymmetry, inverse, composition and subrelation relations.

\section{Background and Notation}

Given a knowledge graph that is represented as a list of fact triples, knowledge graph embedding methods define scoring function to measure the plausibility of these triples.
We denote a triple by $(h, r, t) $, where $h $ represents head entity, $r $ represents relation and $t $ represents tail entity.
The column vectors of entities and relations are represented by bold lower case letters, which belong to set $\mathcal{E} $ and $\mathcal{R} $ respectively. We denote the set of all triples that are true in a world as $\mathcal{T} $. $f_r(h, t)$ represents the scoring function.

We take the definition of complex relations from \cite{wang2014knowledge}.
For each relation $r$, we compute average number of tails per head (tphr) and average number of heads per tail (hptr). If tphr $<$ 1.5 and hptr $<$ 1.5, $r$ is treated as 1-to-1; if tphr $>$ 1.5 and hptr $>$ 1.5, $r$ is treated as a N-to-N; if tphr $>$ 1.5 and hptr $<$ 1.5, $r$ is treated as 1-to-N.

We focus on four important relation patterns, which includes: (1) \textbf{Symmetry/antisymmetry}. A relation $r $ is symmetric
if $\forall e_1, e_2 \in \mathcal{E}, (e_1, r, e_2) \in \mathcal{T} \iff
 (e_2, r, e_1) \in \mathcal{T}$ and is antisymmetric
 if $(e_1, r, e_2) \in \mathcal{T} \Rightarrow (e_2, r, e_1) \notin
 \mathcal{T}$; (2) \textbf{Inverse}. If $\forall e_1, e_2 \in \mathcal{E}, (e_1, r_1, e_2) \in \mathcal{T} \iff (e_2, r_2, e_1) \in \mathcal{T}$, then $r_1 $
and $r_2 $ are inverse relations; (3) \textbf{Composition}.
If $\forall e_1, e_2, e_3 \in \mathcal{E}, (e_1, r_1, e_2) \in \mathcal{T} \land (e_2, r_2, e_3) \in \mathcal{T} \Rightarrow (e_1, r_3, e_3) \in \mathcal{T}$,
then $r_3 $ can be seen as the composition of $r_1$ and $r_2 $;
(4) \textbf{Subrelation} \cite{qu2019probabilistic}.
If $\forall e_1, e_2 \in \mathcal{E}, (e_1, r_1, e_2) \in \mathcal{T}  \Rightarrow (e_1, r_2, e_2) \in \mathcal{T}$,
then $r_2 $ can be seen as a subrelation of $r_1$.

\begin{table*}[h]
\begin{center}
\resizebox{0.7\textwidth}{!}{
\begin{tabular}{c|c|c|ccccc}
\hline
\multirow{2}{*}{Method} &\multirow{2}{*}{Score Function}
&\multirow{2}{*}{\begin{tabular}[c]{@{}c@{}}Performance of\\ complex relations\end{tabular}}
&\multicolumn{5}{c}{Relation Patterns} \\
&          &     & $Sym$     & $Asym$     & $Inv$     & $Comp$   &$Sub$ \\ \hline
TransE  &$-||\bm{h} + \bm{r} - \bm{t}||$  &Low  &\xmark  &\cmark   &\cmark  &\cmark  &\xmark \\
TransR  &$-||\bm{M}_{r}\bm{h} + \bm{r} - \bm{M}_{r}\bm{t}||$  &High &\xmark  &\cmark   &\xmark  &\xmark &\xmark\\
RotatE &$-||\bm{h}\circ\bm{r} - \bm{t}||$ &Low  &\cmark  &\cmark   &\cmark  &\cmark  &\xmark   \\ \hline
PairRE  &$-||\bm{h}\circ\bm{r}^H - \bm{t}\circ\bm{r}^T||$ &High   &\cmark  &\cmark   &\cmark  &\cmark &{\cmark}* \\ \hline
\end{tabular}
}
\end{center}
\caption{\label{table:Comparision} Comparison between PairRE and some distance based embedding methods.
$Sym$, $Asym$, $Inv$, $Comp$ and $Sub$ are abbreviations for symmetry, antisymmetry, inverse and subrelation respectively.
{\cmark}* means the model can have the specific capability with some constraints.}
\end{table*}

\section{Related Work}
\textbf{Distance based models}.
Distance based models measure plausibility of fact triples as distance between entities.
TransE interprets relation as a translation vector $r $ so that entities can be connected, i.e., $h + r \approx t$.
TransE is efficient, though cannot model symmetry relations and have difficulty in modeling complex relations.
Several models are proposed for improving TransE to deal with complex relations, including TransH, TransR, TransD, TranSparse \cite{ji2016knowledge} and so on.
All these methods project the entities to relation specific hyperplanes or spaces first, then translate projected entities with relation vectors.
By projecting entities to different spaces or hyperplanes, the ability to handle complex relations is improved.
However, with the added projection parameters, these models are unable to encode inverse and composition relations.

The recent state-of-the-art, RotatE, which can encode symmetry/antisymmetry, inverse and composition relation patterns, utilizes rotation based translational method in a complex space. Although expressiveness for different relation patterns, complex relations remain challenging. 
GC-OTE \cite{tang2019orthogonal} proposes to improve complex relation modeling ability of RotatE by introducing graph context to entity embedding.  
However, the calculation of graph contexts for head and tail entities is time consuming, which is inefficient for large scale knowledge graphs, e.g. ogbl-wikikg \cite{hu2020open}.

Another related work is SE \cite{bordes2011learning}, which utilizes two separate relation matrices to project head and tail entities.
As pointed out by \cite{sun2019rotate}, this model is not able to encode symmetry/antisymmetry, inverse and composition relations.

Table \ref{table:Comparision} shows comparison between our method and some representative distance based methods.
As the table shows, our model is the most expressive one, with the ability to handle complex relations and encode four key relation patterns. 

\textbf{Semantic matching models}.
Semantic matching models exploit similarity-based scoring functions, which can be divided into bilinear models and neural network based models.
As the models have been developed, such as RESCAL \cite{nickel2011three}, DistMult \cite{yang2014embedding}, HolE \cite{nickel2016holographic}, ComplEx \cite{trouillon2016complex} and  QuatE \cite{zhang2019quaternion}, the key relation encoding abilities are enriched.  However, all these models have the flaw in encoding composition relations \cite{sun2019rotate}.

RESCAL, ComplEx and SimplE \cite{kazemi2018simple} are all proved to be fully expressive when embedding dimensions fulfill some requirements \cite{wang2018multi, trouillon2016complex, kazemi2018simple}.
The fully expressiveness means these models can express all the ground truth which exists in the data, including complex relations.
However, these requirements are hardly fulfilled in practical use.
It is proved by \cite{wang2018multi} that, to achieve complete expressiveness, the embedding dimension should be greater than $N $/32, where $N $ is the number of entities in dataset.

Neural networks based methods, e.g., convolution neural networks \cite{dettmers2018convolutional}, graph convolutional networks \cite{schlichtkrull2018modeling}  show promising performances.
However, they are difficult to analyze as they work as a black box.

\textbf{Encoding Subrelation}.
Existing methods encode subrelation by utilizing first order logic rules.
One way is to augment knowledge graphs with grounding of rules, including subrelation rules \cite{guo2018knowledge, qu2019probabilistic}.
The other way is adding constraints on entity and relation representations, e.g., ComplEx-NNE-AER and SimplE$^+$.
The second way enriches the models' expressiveness with relatively low cost. In this paper, we show that PairRE can  encode subrelation with constraints on relation representations while keeping the ability to encode symmetry/antisymmetry, inverse and composition relations.

\begin{figure*}[h]
   \begin{subfigure}{.33\textwidth}
    \centering
    \includegraphics[width=0.6\linewidth]{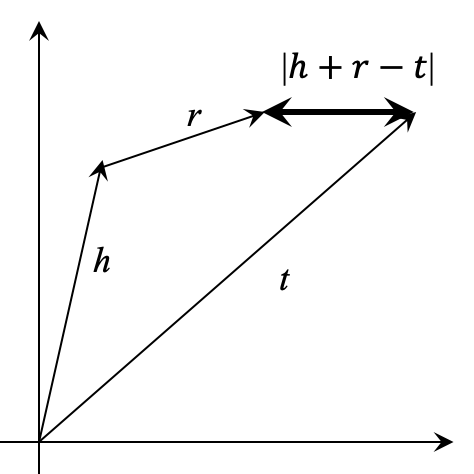}
    \caption{TransE}
    \label{fig:transe}
  \end{subfigure}
  \begin{subfigure}{.33\textwidth}
    \centering
    \includegraphics[width=0.6\linewidth]{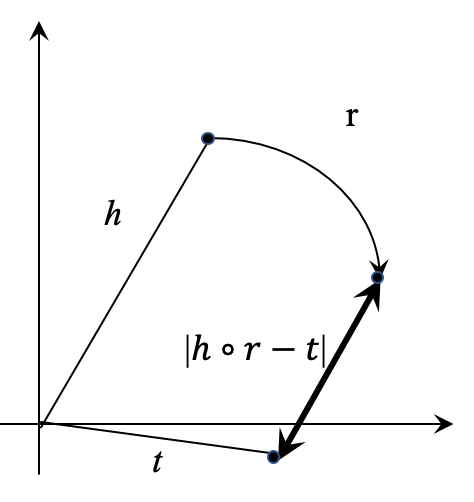}
    \caption{RotatE}
    \label{fig:rotate}
  \end{subfigure}
  \begin{subfigure}{.33\textwidth}
    \centering
    \includegraphics[width=0.65\linewidth]{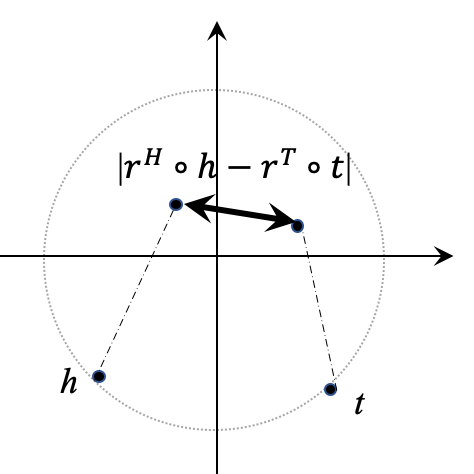}
    \caption{PairRE}
    \label{fig:pairre}
  \end{subfigure}
\caption{Illustration of TransE, RotatE and PairRE when the entities stay in a plane.
For PairRE, all entities are on the unit circle.
The relation vectors project entities to different locations.
}
\label{fig:Illustration}
\end{figure*}

\section{Methodology}\label{Method}

To overcome the problem of modeling 1-to-N/N-to-1/N-to-N complex relations and enrich the capabilities for different relation patterns, we propose a model with paired vectors for each relation.
Given a training triple ($h $, $r $, $t $), our model learns vector embeddings of entities and relation in real space.
Specially, PairRE takes relation embedding as paired vectors, which is represented as $[\bm{r}^H, \bm{r}^T]$.
$\bm{r}^H$ and $\bm{r}^T$ project head entity $h$ and tail entity $t$ to Euclidean space respectively.
The projection operation is the Hadamard product\footnote{Hadamard product means entry-wise product.} between these two vectors.
PairRE then computes distance of the two projected vectors as plausibility of the triple .
We want that $\bm{h} \circ \bm{r}^H \approx \bm{t} \circ \bm{r}^T $ when ($h, r, t$) holds, while $\bm{h} \circ \bm{r}^H$ should be far away from $\bm{t} \circ \bm{r}^T$ otherwise.
In this paper, we take the $L_1$-norm to measure the distance.

In order to remove scaling freedoms, we also add constraint on embeddings similar to previous distance based models \cite{bordes2013translating,wang2014knowledge, lin2015learning}. And the constraint is only added on entity embeddings. We want relation embeddings to capture semantic connection among relation vectors (e.g., $PeopleBornHere$ and $PlaceOfBirth$) and complex characteristic (e.g., 1-N) easily and sufficiently. For entity embedding, the $L_2$-norm is set to be 1.

The scoring function is defined as follows:
\begin{equation}
f_r(\bm{h}, \bm{t}) = -{||\bm{h}\circ\bm{r}^H - \bm{t}\circ\bm{r}^T||},
\end{equation}
where $\bm{h}, \bm{r}^H, \bm{r}^T, \bm{t} \in \mathbb{R}^d $ and ${||\bm{h}||}^2 = {||\bm{t}||}^2 = 1$.
The model parameters are, all the entities’ embeddings, $\{\bm{e}_j\}_{j=1}^{\mathcal{E}} $ and all the relations’ embeddings,
$\{\bm{r}_j\}_{j=1}^{\mathcal{R}} $.

Illustration of the proposed PairRE is shown in Figure~\ref{fig:Illustration}.
Compared to TransE/RotatE, PairRE enables an entity to have distributed representations when involved in different relations. We also find the paired relation vectors enable an adaptive adjustment of the margin in loss function, which alleviates the modeling problem for complex relations. 

Let's take a 1-to-N relation as an example. We set the embedding dimension to one and remove the constraint on entity embeddings for better illustration.
Given triples $(h, r, ?)$, where the correct tail entities belong to set $S = \{t_1, t_2, ..., t_N\}$, PairRE predicts tail entities by letting 
\begin{equation}
||\bm{h} \circ \bm{r}^H - \bm{t}_i \circ \bm{r}^T|| < \gamma,
\nonumber
\end{equation}
where $\gamma$ is a fixed margin for distance based embedding models and $t_i \in S$.
The value of $\bm{t}_i$ should stay in the following range: 
\begin{equation}
\resizebox{.997\hsize}{!}{
$
\bm{t}_i\in 
\begin{cases}
((\bm{h}\circ\bm{r}^H-\gamma)/\bm{r}^T, (\bm{h}\circ\bm{r}^H+\gamma)/\bm{r}^T),\text{if } \bm{r}^T>0, \\ 
((\bm{h}\circ\bm{r}^H+\gamma)/\bm{r}^T, (\bm{h}\circ\bm{r}^H-\gamma)/\bm{r}^T),\text{if } \bm{r}^T<0, \\
(-\infty, +\infty), \text{otherwise}.
\end{cases}
\nonumber
$}
\end{equation}
The above analysis shows PairRE can adjust the value of $\bm{r}^T$ to fit the entities in  $S$. The larger the size of $S$, the smaller the absolute value $\bm{r}^T$. While models like TransE or RotatE have a fixed margin for all complex relation types. When the size of $S$ is large enough, these models will be difficult to fit the data. For N-to-1 relations, PairRE can also adjust the value of $\bm{r}^H$ adaptively to fit the data.

Meanwhile, not adding a relation specific translational vector enables the model to encode several key relation patterns. We show these capabilities below.

\begin{proposition}
PairRE can encode symmetry/antisymmetry relation pattern.
\end{proposition}

\begin{proof}
If $ (e_1, r_1, e_2) \in \mathcal{T}$ and $ (e_2, r_1, e_1) \in \mathcal{T}$, we have
\begin{equation}
\begin{aligned}
\bm{e}_1 \circ \bm{r}_1^H = \bm{e}_2 \circ \bm{r}_1^T \land \bm{e}_2 \circ \bm{r}_1^H = \bm{e}_1 \circ \bm{r}_1^T \\ \Rightarrow {\bm{r}_1^H}^2 = {\bm{r}_1^T}^2
  \end{aligned}
\end{equation}
if $ (e_1, r_1, e_2) \in \mathcal{T}$ and $ (e_2, r_1, e_1) \notin \mathcal{T}$, we have
\begin{equation}
\begin{aligned}
\bm{e}_1 \circ \bm{r}_1^H = \bm{e}_2 \circ \bm{r}_1^T \land \bm{e}_2 \circ \bm{r}_1^H \neq \bm{e}_1 \circ \bm{r}_1^T \\ \Rightarrow {\bm{r}_1^H}^2 \neq {\bm{r}_1^T}^2
\end{aligned}
\end{equation}
\end{proof}

\begin{proposition}
PairRE can encode inverse relation pattern.
\end{proposition}

\begin{proof}
If $ (e_1, r_1, e_2) \in \mathcal{T}$ and $ (e_2, r_2, e_1) \in \mathcal{T}$, we have
\begin{equation}
\begin{aligned}
\bm{e}_1 \circ \bm{r}_1^H = \bm{e}_2 \circ \bm{r}_1^T \land \bm{e}_2 \circ \bm{r}_2^H = \bm{e}_1 \circ \bm{r}_2^T \\ \Rightarrow {\bm{r}_1^H} \circ\ {\bm{r}_2^H}  = {\bm{r}_1^T} \circ {\bm{r}_2^T}
\end{aligned}
\end{equation}
\end{proof}

\begin{proposition}
PairRE can encode composition relation pattern.
\end{proposition}

\begin{proof}
If $ (e_1, r_1, e_2) \in \mathcal{T}$, $ (e_2, r_2, e_3) \in \mathcal{T}$ and $(e_1, r_3, e_3) \in \mathcal{T} $, we have
\begin{equation}
  \begin{aligned}
 \bm{e}_1 \circ \bm{r}_1^H = \bm{e_2} \circ \bm{r}_1^T  \land  \bm{e}_2 \circ \bm{r}_2^H = \bm{e_3} \circ \bm{r}_2^T  \land  &\\ \bm{e}_1 \circ \bm{r}_3^H = \bm{e_3} \circ \bm{r}_3^T
 \\ \Rightarrow \bm{r}_1^T \circ \bm{r}_2^T \circ \bm{r}_3^H = \bm{r}_1^H \circ  \bm{r}_2^H  \circ  \bm{r}_3^T
  \end{aligned}
\end{equation}
\end{proof}

Moreover, with some constraint, PairRE can also encode subrelations. For a subrelation pair, $\forall h, t \in \mathcal{E}$ : $(h, r_1, t) \rightarrow (h, r_2, t)$, it suggests triple $(h, r_2, t)$ should be always more plausible than triple $(h, r_1, t)$. 
In order to encode this pattern, PairRE should have the capability to enforce $f_{r_2}(h, r_2, t) \geq f_{r_1}(h, r_1, t)$.

\begin{proposition}
PairRE can encode subrelation relation pattern using inequality constraint.
\end{proposition}

\begin{proof}
Assume a subrelation pair $r_1$ and $r_2$ that $\forall h, t \in \mathcal{E}$: $(h, r_1,t) {\rightarrow} (h, r_2, t)$.
We impose the following constraints:
\begin{equation}
\frac{\bm{r}_{2, i}^H}{\bm{r}_{1,i}^H} = \frac{\bm{r}_{2,i}^T}{\bm{r}_{1,i}^T} = \bm{\alpha}_i, |\bm{\alpha}_i| \leq 1,
\end{equation}
where $\alpha \in \mathbb{R}^d$. Then we can get
\begin{equation}
\resizebox{.995\hsize}{!}{
$
  \begin{aligned}
&f_{r_2}(h, t) - f_{r_1}(h, t) \\
&= || \bm{h} \circ \bm{r}_1^H - \bm{t} \circ \bm{r}_1^T || -  || \bm{h} \circ \bm{r}_2^H - \bm{t} \circ \bm{r}_2^T || \\
&= ||\bm{h} \circ \bm{r}_1^H - \bm{t} \circ \bm{r}_1^T ||- ||\bm{\alpha} \circ (\bm{h} \circ \bm{r}_1^H - \bm{t} \circ \bm{r}_1^T)|| \\
& \geq 0.
  \end{aligned}
  $}
\end{equation}
When the constraints are satisfied, PairRE forces triple $(h, r_2, t)$ to be more plausible than triple $(h, r_1, t)$.
\end{proof}

\textbf{Optimization}. To optimize the model, we utilize the self-adversarial negative sampling loss \cite{sun2019rotate} as objective for training:
\begin{equation}
  \begin{aligned}
L = & -\log{\sigma(\gamma - f_r(\bm{h}, \bm{t}))} \\
 &- \sum_{i=1}^{n}{p(h_{i}^{'}, r, t_{i}^{'})}\log{\sigma(f_r(\bm{h_{i}^{'}}, \bm{t_{i}^{'}}) - \gamma)},
 \end{aligned}
\label{Eq:loss}
\end{equation}
where $\gamma $ is a fixed margin and $\sigma$ is the sigmoid function.
($h_{i}^{'} $, $r $, $t_{i}^{'} $) is the $i^{th}$ negative
triple and $p(h_{i}^{'}, r, t_{i}^{'}) $ represents the weight of this negative
sample. $p(h_{i}^{'}, r, t_{i}^{'}) $ is defined as
follows:
\begin{equation}
  \begin{aligned}
p((h_{i}^{'}, r, t_{i}^{'})|(h, r, t)) = \frac{{\exp{ f_r(h_{i}^{'}, t_{i}^{'})}}}{{\sum_j{\exp{
f_r(h_{j}^{'}, t_{j}^{'})}}}}.
  \end{aligned}
\end{equation}

\section{Experimental results}
\subsection{Experimental setup}
We evaluate the proposed method on link prediction tasks.
At first, we validate the ability to deal with complex relations and symmetry/antisymmetry, inverse and composition relations on four benchmarks.
Then we validate our model on two subrelation specific benchmarks. Statistics of these benchmarks are shown in Table ~\ref{table:dataset}.

\textbf{ogbl-wikikg2}\footnote{ogbl-wikikg2 fixes a bug in test/validation negative samples from original ogbl-wikikg.} \cite{hu2020open} is extracted from Wikidata knowledge base \cite {vrandevcic2014wikidata}. One of the main challenges for this dataset is complex relations. \textbf{ogbl-biokg} \cite{hu2020open} contains data from a large number of biomedical data repositories. One of the main challenges for this dataset is symmetry relations. \textbf{FB15k} \cite{bordes2013translating} contains triples from Freebase.
The main relation patterns are inverse and symmetry/antisymmetry.
\textbf{FB15k-237} \cite{toutanova2015observed} is a subset of FB15k, with inverse relations removed.
The main relation patterns are antisymmetry and composition.
\textbf{DB100k} \cite{ding2018improving} is a subset of DBpedia.
The main relation patterns are composition, inverse and subrelation.
\textbf{Sports} \cite{wang2015knowledge} is a subset of NELL \cite{mitchell2018never}.
The main relation patterns are antisymmetry and subrelation.

%\begin{wraptable}{r}{7cm}
\begin{table}[t]
\centering
\resizebox{0.45\textwidth}{!}{
\begin{tabular}{c|c|c|c|c|c}
\hline
\textbf{Dataset} &\bm{$|\mathcal{R}|$} & \bm{$|\mathcal{E}|$} & \textbf{Train} & \textbf{Valid} & \textbf{Test} \\ \hline
ogbl-wikikg2 &535 &2,500k  &16,109k &429k &598k \\ \hline
ogbl-biokg  &51 &94k  &4,763k &163k &163k \\ \hline
FB15k &13k &15k &483k &50k &59k  \\ \hline
FB15k-237 &237 &15k &272k &18k &20k \\ \hline
DB100k &470 &100k  &598k &50k &50k \\ \hline
Sports &4  &1039 &1312 &- &307 \\\hline
\end{tabular}
}
\label{table:datasets}
\caption{\label{table:dataset} Number of entities, relations, and observed triples in each split for the six benchmarks.}
\end{table}
%\end{wraptable}

\begin{comment}
\begin{table*}[h]
\begin{center}
\resizebox{0.8\textwidth}{!}{
\begin{tabular}{c|ccc|ccc}
\hline
- &\multicolumn{3}{c|}{ogbl-wikikg2} & \multicolumn{3}{c}{ogbl-biokg} \\ \hline
\textbf{Model} &$\#$Dim &Test MRR &Valid MRR &$\#$Dim &Test MRR &Valid MRR  \\ \hline
TransE &100 &$0.2535\pm0.004$ &$0.4587\pm0.003$ &- &- &- \\
DistMult &100 &$0.3434\pm0.008$ &$0.3412\pm0.007$ &- &- &- \\
ComplEx &50 &$0.3877\pm0.005$ &$0.3612\pm0.006$ &- &- &-  \\
RotatE &50 &$0.2681\pm0.005$  &$0.3613\pm0.003$ &- &- &-	\\ \hline
PairRE &100 &$\textbf{0.4912}\pm\textbf{0.004}$ &$\textbf{0.5013}\pm\textbf{0.004}$ &- &- &- \\  \Xhline{1.2pt}
TransE &600$\dagger$ &$0.4536\pm0.003$ &$0.4587\pm0.003$ &2000 &$0.7452\pm0.0004$	&$0.7456\pm0.0003$ \\
DistMult &600$\dagger$  &$0.3612\pm0.003$	&$0.3403\pm0.001$ &2000 &$0.8043\pm0.0003$	&$0.8055\pm0.0003$ \\
ComplEx &300$\dagger$  &$0.4028\pm0.003$	 &$0.3787\pm0.004$ &1000 &$0.8095\pm0.0007$	&$0.8105\pm0.0001$ \\
RotatE &300$\dagger$  &$0.3626\pm0.004$	&$0.3613\pm0.003$ &1000 &$0.7989\pm0.0004$	&$0.7997\pm0.0002$ \\ \hline
PairRE &200 &$\textbf{0.5289}\pm\textbf{0.003}$ &$\textbf{0.5529}\pm\textbf{0.001}$ &2000  &$\textbf{0.8164}\pm\textbf{0.0005}$ &$\textbf{0.8172}\pm\textbf{0.0005}$ \\ \hline
\end{tabular}
}
\end{center}
\caption{\label{table:ogbl} Link prediction results on ogbl-wikikg2 and ogbl-biokg.
Best results are in bold. All the results except PairRE are from \cite{hu2020open}.
$\dagger$ requires a GPU with 48GB memory.
PairRE runs on a GPU with 16GB memory.}
\end{table*}
\end{comment}

\begin{table*}[h]
\begin{center}
\resizebox{0.78\textwidth}{!}{
\begin{tabular}{c|ccc|ccc}
\hline
- &\multicolumn{3}{c|}{ogbl-wikikg2} & \multicolumn{3}{c}{ogbl-biokg} \\ \hline
\textbf{Model} &$\#$Dim &Test MRR &Valid MRR &$\#$Dim &Test MRR &Valid MRR  \\ \hline
TransE &100 &$0.2622\pm0.0045$ &$0.2465\pm0.0020$ &- &- &- \\
DistMult &100 &$0.3447\pm0.0082$ &$0.3150\pm0.0088$ &- &- &- \\
ComplEx &50 &$0.3804\pm0.0022$ &$0.3534\pm0.0052$ &- &- &-  \\
RotatE &50 &$0.2530\pm0.0034$  &$0.2250\pm0.0035$ &- &- &-	\\ \hline
PairRE &100 &$\textbf{0.4849}\pm0.0029$ &$\textbf{0.4941}\pm0.0035$ &- &- &- \\  \Xhline{1.2pt}
TransE &500$\dagger$ &$0.4256\pm0.0030$ &$0.4272\pm0.0030$ &2000 &$0.7452\pm0.0004$	&$0.7456\pm0.0003$ \\
DistMult &500$\dagger$  &$0.3729\pm0.0045$	&$0.3506\pm0.0042$ &2000 &$0.8043\pm0.0003$	&$0.8055\pm0.0003$ \\
ComplEx &250$\dagger$  &$0.4027\pm0.0027$ &$0.3759\pm0.0016$ &1000 &$0.8095\pm0.0007$	&$0.8105\pm0.0001$ \\
RotatE &250$\dagger$  &$0.4332\pm0.0025$	&$0.4353\pm0.0028$ &1000 &$0.7989\pm0.0004$	&$0.7997\pm0.0002$ \\ \hline
PairRE &200 &$\textbf{0.5208}\pm0.0027$ &$\textbf{0.5423}\pm0.0020$ &2000  &$\textbf{0.8164}\pm0.0005$ &$\textbf{0.8172}\pm0.0005$ \\ \hline
\end{tabular}
}
\end{center}
\caption{\label{table:ogbl} Link prediction results on ogbl-wikikg2 and ogbl-biokg.
Best results are in bold. All the results except PairRE are from \cite{hu2020open}.
$\dagger$ requires a GPU with 48GB memory.
PairRE runs on a GPU with 16GB memory.}
\end{table*}

\begin{table*}[h]
\begin{center}
\resizebox{0.85\textwidth}{!}{
\begin{tabular}{c|ccccc|ccccc}
\hline
- &\multicolumn{5}{c|}{FB15k} & \multicolumn{5}{c}{FB15k-237} \\ \hline
\textbf{Model} & MR & MRR &Hit@10 &Hit@3 &Hit@1 &MR &MRR &Hit@10 &Hit@3 &Hit@1 \\ \hline
TransE$\dag$ &- &0.463 &0.749 &0.578 &0.297 &357 &0.294 &0.465 &- &-    \\
DistMult$\Diamond$ &42 &0.798 &0.893 &- &-  &254 &0.241 &0.419 &0.263 &0.155   \\
HolE   &- &0.524 &0.739 &0.759 &0.599 &- &- &- &- &-                \\
ConvE &51 &0.657 &0.831 &0.723 &0.558 &244 &0.325 &0.501 &0.356 &0.237          \\
ComplEx &- &0.692 &0.840 &0.759 &0.599 &339 &0.247 &0.428 &0.275 &0.158        \\
SimplE &- &0.727 &0.838 &0.773 &0.660 &- &- &- &- &-         \\
RotatE &40 &0.797 &0.884 &0.830 &0.746 &177 &0.338 &0.533 &0.375 &0.241  \\
SeeK &- &\textbf{0.825} &0.886 &0.841 &\textbf{0.792} &- &- &- &- &- \\
OTE &- &- &- &- &- &- &0.351 &0.537 &0.388 &0.258 \\
GC-OTE  &- &- &- &- &- &- &\textbf{0.361} &\textbf{0.550} &\textbf{0.396} &\textbf{0.267} \\ \hline
PairRE &\textbf{37.7} &0.811 &\textbf{0.896} &\textbf{0.845} &0.765  &$\bm{160}$  &0.351  &$0.544$ &$0.387$  &$0.256$ \\
  &$\pm0.4979$ &$\pm0.00077$ &$\pm0.00071$  &$\pm0.0011$ &$\pm0.0012$   &$\pm0.9949$  &$\pm0.00066$ &$\pm0.00093$
&$\pm0.00079$ &$\pm0.00097$
 \\  \hline
\end{tabular}
}
\end{center}
\caption{\label{table:FB15k} Link prediction results on FB15k and FB15k-237. Results of $[\dag]$ are taken from \cite{nickel2016holographic}; Results of $[\Diamond]$ are taken from \cite{kadlec2017knowledge}. Other results are taken from the corresponding papers. GC-OTE adds graph context to OTE \cite{tang2019orthogonal}.}
\end{table*}

\begin{comment}
\begin{itemize}
\item[$\bullet$]
\textbf{ogbl-wikikg2}\footnote{ogbl-wikikg2 fixes a bug in test/validation negative samples from original ogbl-wikikg.} \cite{hu2020open} is extracted from Wikidata knowledge base \cite {vrandevcic2014wikidata}. One of the main challenges for this dataset is complex relations.
\end{itemize}
\begin{itemize}
\item[$\bullet$]
\textbf{ogbl-biokg} \cite{hu2020open} contains data from a large number of biomedical data repositories. One of the main challenges for this dataset is symmetry relations.
\end{itemize}
\begin{itemize}
\item[$\bullet$]
\textbf{FB15k} \cite{bordes2013translating} contains triples from Freebase.
The main relation patterns are inverse and symmetry/antisymmetry.
\end{itemize}
\begin{itemize}
\item[$\bullet$]
\textbf{FB15k-237} \cite{toutanova2015observed} is a subset of FB15k, with inverse relations removed.
The main relation patterns antisymmetry and composition.
\end{itemize}
\begin{itemize}
\item[$\bullet$]
\textbf{DB100k} \cite{ding2018improving} is a subset of DBpedia.
The main relation patterns are composition and subrelation.
\end{itemize}
\begin{itemize}
\item[$\bullet$]
\textbf{Sports} \cite{wang2015knowledge} is a subset of NELL \cite{mitchell2018never}.
The main relation patterns are antisymmetry and subrelation.
\end{itemize}
\end{comment}

\textbf{Evaluation protocol}. Following the state-of-the-art methods, we measure the quality of the ranking of each test triple among all possible head entity and tail entity substitutions: ($h^{'} $, $r $ , $t $) and ($h $, $r $, $t^{'} $), $\forall h^{'} $, $\forall t^{'} \in \mathcal{E} $.
Three evaluation metrics, including Mean Rank(MR), Mean Reciprocal Rank (MRR) and Hit ratio with cut-off values $n$ = 1, 3, 10, are utilized.
MR measures the average rank of all correct entities.
MRR is the average inverse rank for correct entities with higher value representing better performance.
Hit@$n$ measures the percentage of correct entities in the top $n$ predictions.
The rankings of triples are computed after removing all the other observed triples that appear in either training, validation or test set. For experiments on ogbl-wikikg2 and ogbl-biokg, we follow the evaluation protocol of these two benchmarks \cite{hu2020open}.

\textbf{Implementation}.
We utilize the official implementations of benchmarks ogbl-wikikg2 and ogbl-biokg \cite{hu2020open} for the corresponding experiments\footnote{Our code is available at: \tiny{https://github.com/alipay/KnowledgeGraphEmbeddingsViaPairedRelationVectors\_PairRE}}.
Only the hypeparameter $\gamma$ and embedding dimension are tuned.
The other settings are kept the same with baselines.
For the rest experiments, we implement our models based on the implementation of RotatE \cite{sun2019rotate}.
All hypeparameters except $\gamma$ and embedding dimension are kept the same with RotatE.

\begin{table}[t]
\centering
\resizebox{0.45\textwidth}{!}{
\begin{tabular}{c}
\hline
Subrelation \\ \hline
(h, CoachesTeam, t) $\rightarrow$ (h, PersonBelongsToOrganization, t) \\
(h, AthleteLedSportsTeam, t) $\rightarrow$ (h, AtheletePlaysForTeam, t) \\ \hline
\end{tabular}
}
\caption{\label{table:sports_subrelation}  
The added subrelation rules for Sports dataset.}
\end{table}

\begin{table}[t]
\centering
\resizebox{0.35\textwidth}{!}{
\begin{tabular}{c|c|c}
\hline
Model  &MRR   &hit@1 \\ \hline
SimplE &0.230 &0.184 \\ 
SimplE$^+$   &0.404  &0.349 \\ \hline
PairRE  &0.468  $\pm $ 0.003  &0.416 $\pm $ 0.005  \\
PairRE+Rule &$\textbf{0.475}$  $\pm $ 0.003  &$\textbf{0.432}$ $\pm $ 0.004 \\ \hline
\end{tabular}
}
\caption{\label{table:sports_weight_tying} Link prediction results on Sports dataset.
Other results are taken from \cite{fatemi2019improved}.}
\end{table}

\begin{table}[h]
\centering
\resizebox{0.45\textwidth}{!}{
\begin{tabular}{c|cccc}
\hline
%- &\multicolumn{4}{c}{DB100K} \\ \hline
\textbf{Model} &MRR &Hit@10 &Hit@3 &Hit@1  \\ \hline
TransE &0.111 &0.270 &0.164 &0.016 \\
DistMult &0.233 &0.448 &0.301 &0.115 \\
HolE &0.260 &0.411 &0.309 &0.182 \\
ComplEx &0.242 &0.440 & 0.312 &0.126 \\
SeeK &0.338 &0.467 &0.370 & 0.268 \\ \hline
ComplEx-NNE &0.298 &0.426 &0.330 &0.229 \\
ComplEx-NNE-AER &0.306 &0.418 &0.334 &0.244 \\ \hline
PairRE &0.412 &$\textbf{0.600}$ &0.472 &0.309 \\
 &$\pm0.0015$ &$\pm0.0006$ &$\pm0.0015$ &$\pm0.0027$ \\ \hline
PairRE+rule &$\textbf{0.419}$ &0.599 &$\textbf{0.475}$ &$\textbf{0.321}$ \\
 &$\pm0.0010$ &$\pm0.0008$ &$\pm0.0008$ &$\pm0.0016$ \\ \hline
\end{tabular}
}
\caption{\label{table:DB100k} Link prediction results on DB100k. All the results are taken from the corresponding papers.}
\end{table}

\begin{comment}
\begin{table*}[t]
\parbox{.45\linewidth}{
\centering
\resizebox{0.45\textwidth}{!}{
\begin{tabular}{c}
\hline
Subrelation \\ \hline
(h, CoachesTeam, t) $\rightarrow$ (h, PersonBelongsToOrganization, t) \\
(h, AthleteLedSportsTeam, t) $\rightarrow$ (h, AtheletePlaysForTeam, t) \\ \hline
\end{tabular}
}
\caption{\label{table:sports_subrelation}  The added subrelation rules for Sports dataset.}
}
\hfill
\parbox{.45\linewidth}{
\centering
\resizebox{0.35\textwidth}{!}{
\begin{tabular}{c|c|c}
\hline
Model  &MRR   &hit@1 \\ \hline
SimplE &0.230 &0.184 \\ 
SimplE$^+$   &0.404  &0.349 \\ \hline
PairRE  &0.468  $\pm $ 0.003  &0.416 $\pm $ 0.005  \\
PairRE+Rule &$\textbf{0.475}$  $\pm $ 0.003  &$\textbf{0.432}$ $\pm $ 0.004 \\ \hline
\end{tabular}
}
\caption{\label{table:sports_weight_tying} Link prediction results on Sports dataset. 
Other results are taken from \cite{fatemi2019improved}.}
}
\end{table*}
\end{comment}

\subsection{Main results}

Comparisons for ogbl-wikikg2 and ogbl-biokg are shown in Table ~\ref{table:ogbl}. On these two large scale datasets, PairRE achieves state-of-the-art performances.
For ogbl-wikikg2 dataset, PairRE performs best on both limited embedding dimension and increased embedding dimension.
With the same number of parameters to ComplEx (dimension 100), PairRE improves Test MRR close to 10\%.
With increased dimension, all models are able to achieve higher MRR on validation and test sets.
Due to the limitation of hardware, we only increase embedding dimension to 200 for PairRE.
PairRE also outperforms all baselines and improves Test MRR 8.7\%.
\textbf{Based on performances of baselines, the performance of PairRE may be improved further if embedding dimension is increased to 500}.
Under the same experiment setting and the same number of parameters, PairRE also outperforms all baselines on ogbl-biokg  dataset.
It improves Test MRR by 0.69\%, which proves the superior ability to encode symmetry relations.

Comparisons for FB15k and FB15k-237 datasets are shown in Table ~\ref{table:FB15k}.
Since our model shares the same hyper-parameter settings and implementation with RotatE, comparing with this state-of-the-art model is fair to show the advantage and disadvantage of the proposed model.
Besides, the comparisons also include several leading methods, such as TransE \cite{bordes2013translating}, DistMult \cite{yang2014embedding}, HolE \cite{nickel2016holographic}, ConvE \cite{dettmers2018convolutional}, ComplEx \cite{trouillon2016complex}, SimplE \cite{kazemi2018simple}, SeeK \cite{xu2020seek} and OTE \cite{tang2019orthogonal}.
Compared with RotatE, PairRE shows clear improvements on FB15k and FB15k-237 for all evaluation metrics.
For MRR metric, the improvements are 1.4\% and 1.3\% respectively.
Compared with the other leading methods, PairRE also shows highly competitive performances.
All these comparisons prove the effectiveness of PairRE to encode inverse and composition relations.

\subsection{Further experiments on subrelation}

We further compare our method with two of the leading methods ComplEx-NNE-AER and SimplE$^+$, which focus on encoding subrelation.
These two methods add subrelation rules to semantic matching models.
We utilize these rules as constraints on relation representations for PairRE.
Two ways are validated.
We first test the performance of weight tying for subrelation rules on Sports dataset.
The rules ($r_1 {\longrightarrow} r_2 $) are added as follows:
\begin{equation}
  \begin{aligned}
	&\bm{r}_{2}^H = \bm{r}_{1}^H \circ cosine(\bm{\theta}),
	\\&\bm{r}_{2}^T = \bm{r}_{1}^T \circ cosine(\bm{\theta}),
  \end{aligned}
\end{equation}
where $\bm{\theta} \in \mathbb{R}^d$.
The added rules are shown in Table ~\ref{table:sports_subrelation}.
The experiments results in Table ~\ref{table:sports_weight_tying} show effectiveness of the proposed method.

Weight tying on relation representation is a way to incorporate hard rules.
The soft rules can also be incorporated into PairRE by approximate entailment constraints on relation representations.
In this section, we add the same rules from ComplEx-NNE-AER, which includes subrelation and inverse rules.
We denote by $r_1 \stackrel{\lambda}{\longrightarrow} r_2 $ the approximate entailment between relations $r_1$ and $r_2$, with confidence level $\lambda$.
The objective for training is then changed to:
\begin{equation}
\resizebox{.9\hsize}{!}{
$
  \begin{aligned}
L_{rule} & =  L  +  \mu \sum_{\tau_{subrelation}}{\lambda \bm{1}^T(\bm{r}_1^H\circ\bm{r}_2^T - \bm{r}_1^T\circ\bm{r}_2^H )^2} \\& + \mu \sum_{\tau_{inverse}}{\lambda \bm{1}^T(\bm{r}_1^H\circ\bm{r}_2^H - \bm{r}_1^T\circ\bm{r}_2^T )^2},
 \end{aligned}
 $}
\label{Eq:rule_loss}
\end{equation}
where $L$ is calculated from Equation ~\ref{Eq:loss}, $\mu$ is loss weight for added constraints, $\tau_{subrelation}$ and $\tau_{inverse}$ are the sets of subrelation rules and inverse rules respectively.
Following \cite{ding2018improving}, we take the corresponding two relations from subrelation rules as equivalence.
Because $\tau_{subrelation}$ contains both rule $r_1 {\rightarrow} r_2 $ and rule $r_2 {\rightarrow} r_1$.

We validate our method on DB100k dataset. The results are shown in Table ~\ref{table:DB100k}.
We can see PairRE outperforms the recent state-of-the-art SeeK and ComplEx based models with large margins on all evaluation metrics.
With added constraints, the performance of PairRE is improved further.
The improvements for the added rules are 0.7\%, 1.2\% for MRR and Hit@1 metrics respectively.

\begin{figure}[b]
    \centering
    \includegraphics[width=0.69\linewidth]{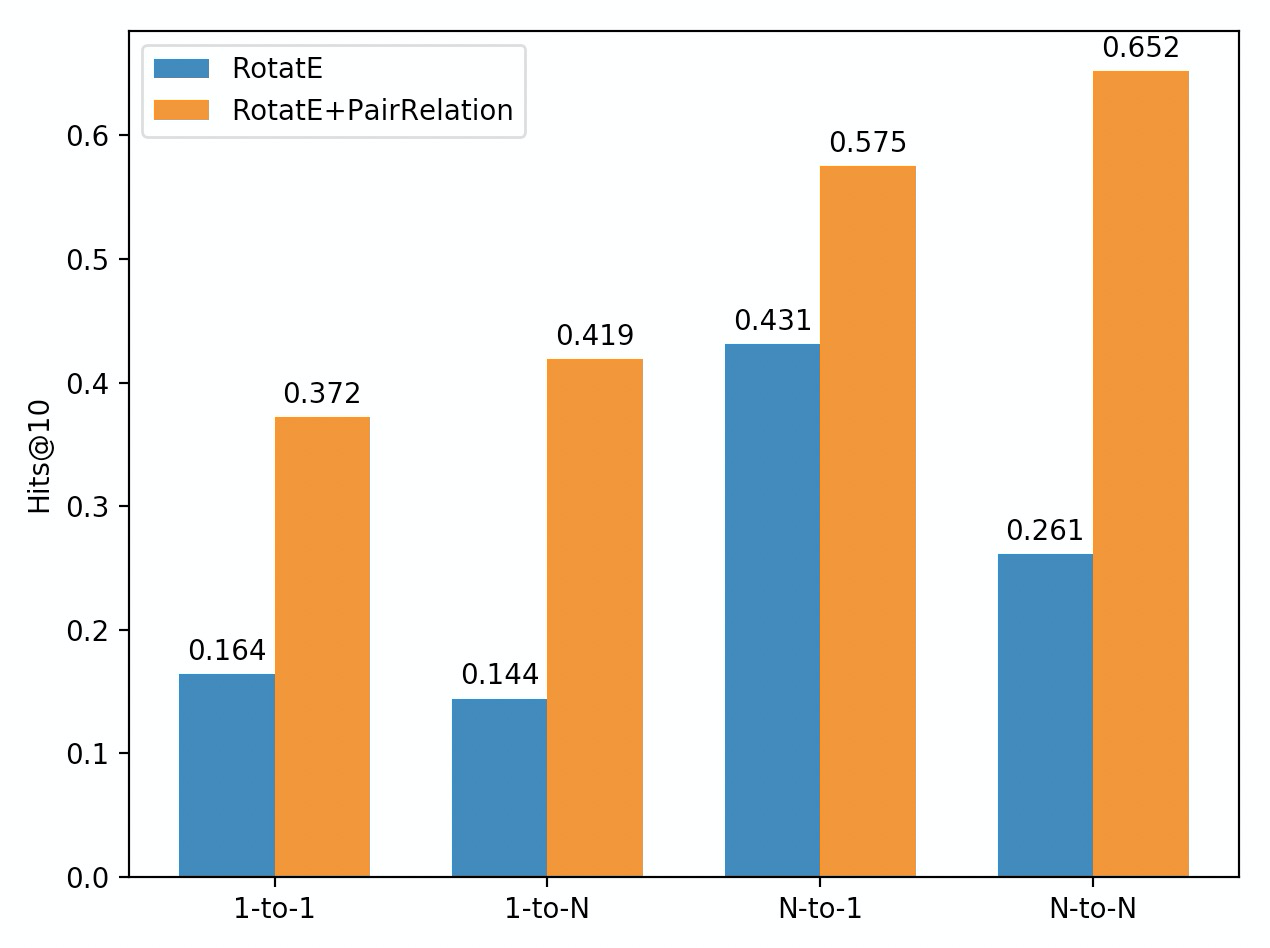}
\caption{Performance comparison between RotatE and RotatE+PairRelation on ogbl-wikikg2 dataset.}
\label{fig:r}
\end{figure}

\begin{table*}[h]
\begin{center}
\resizebox{0.7\textwidth}{!}{
\begin{tabular}{c|cccc|cccc}
\hline
- &\multicolumn{4}{c|}{FB15k(Hits@10)} & \multicolumn{4}{c}{ogbl-wikikg2(Hits@10)} \\ \hline
\textbf{Model} &1-to-1 &1-to-N &N-to-1 &N-to-N &1-to-1 &1-to-N &N-to-1 &N-to-N \\ \hline
KGE2E\_KL\cite{he2015learning} &0.925	&0.813	&$0.802$	&0.715 &- &- &- &- \\
TransE &0.887	&0.822	&0.766	&0.895    &0.074	&0.063	&0.400	&0.220 \\
ComplEx &$\textbf{0.939}$	&0.896	&0.822	&0.902 &$\textbf{0.394}$	&$\textbf{0.278}$	&0.483	&0.504 \\
RotatE &0.923	&0.840	&0.782	&0.908 &0.164	&0.144	&0.431	&0.261 \\ \hline
PairRE &0.785	&$\textbf{0.899}$	&$\textbf{0.872}$	&$\textbf{0.940}$ &0.262	&0.270	&$\textbf{0.594}$	&$\textbf{0.587}$ \\ \hline
%RotatE+PairRelation &- &- &- &-  &$\bm{0.587}$ &$\bm{0.410}$ &$\underline{0.588}$ &$\bm{0.658}$ \\ \hline
\end{tabular}
}
\end{center}
\caption{\label{table:FB15k&&wikikg} Experimental results on FB15k and ogbl-wikikg2 by relation category. Results on FB15k are taken from RotatE \cite{sun2019rotate}.
The embedding dimensions for models on ogbl-wikikg2 are same to the experiments in Table \ref{table:ogbl}, which is 100 for real space models and 50 for complex value based models.}
\end{table*}

\begin{figure*}[!h]
  \begin{subfigure}{.23\textwidth}
    \centering
    \includegraphics[width=1.1\linewidth]{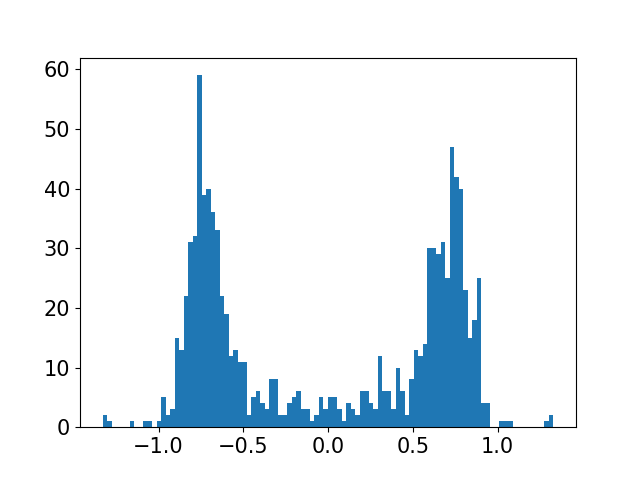}
    \caption{$r_1$}
    \label{fig:embedding:a}
  \end{subfigure}
  \begin{subfigure}{.23\textwidth}
    \centering
    \includegraphics[width=1.1\linewidth]{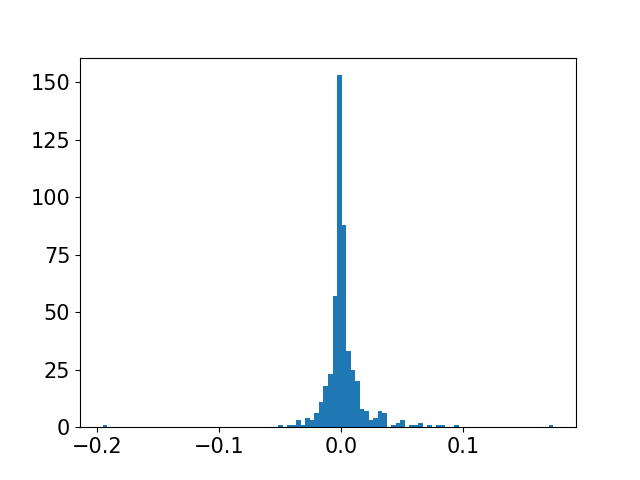}
    \caption{${\bm{r}_{1}^H}^2 - {\bm{r}_{1}^T}^2$}
    \label{fig:embedding:b}
   \end{subfigure}
   \begin{subfigure}{.23\textwidth}
    \centering
    \includegraphics[width=1.1\linewidth]{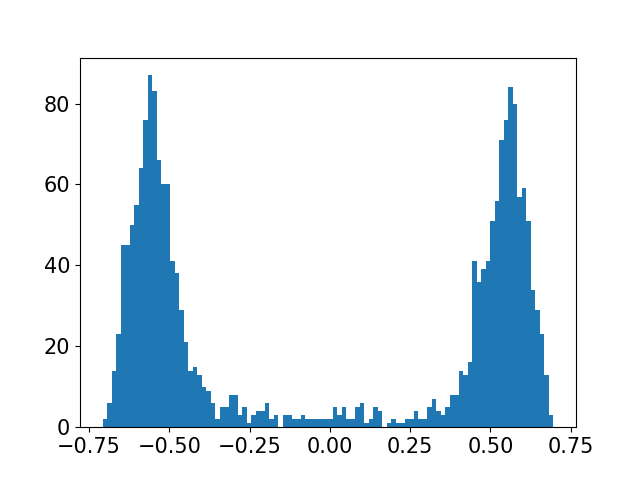}
    \caption{$r_2$}
    \label{fig:embedding:c}
  \end{subfigure}
  \begin{subfigure}{.23\textwidth}
    \centering
    \includegraphics[width=1.1\linewidth]{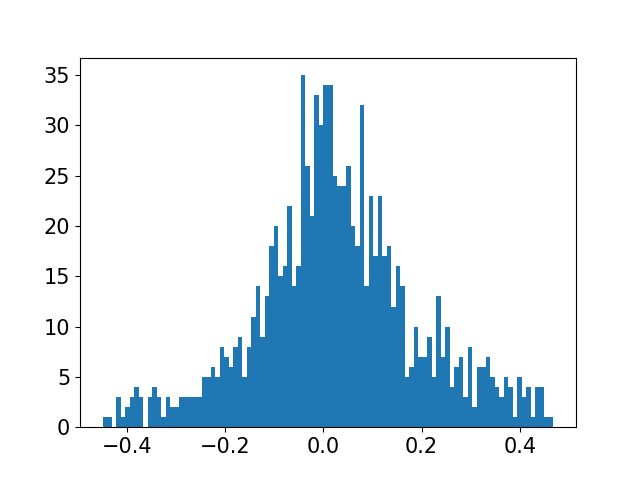}
    \caption{${\bm{r}_{2}^H}^2 - {\bm{r}_{2}^T}^2$}
    \label{fig:embedding:d}
  \end{subfigure}
  
   \begin{subfigure}{.23\textwidth}
    \centering
    \includegraphics[width=1.1\linewidth]{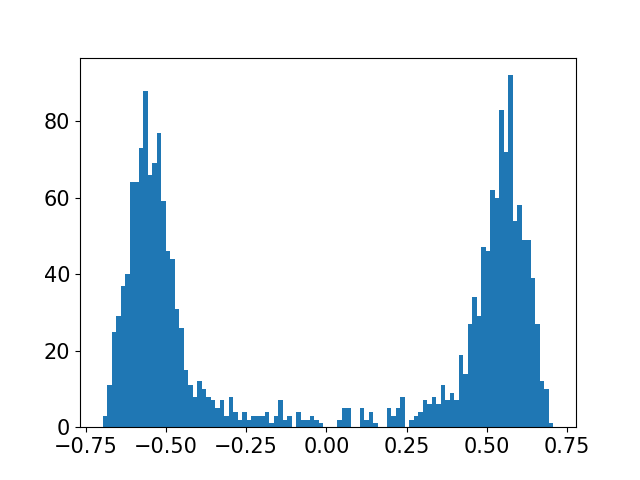}
    \caption{$r_3$}
    \label{fig:embedding:e}
  \end{subfigure}
    \begin{subfigure}{.23\textwidth}
    \centering
    \includegraphics[width=1.1\linewidth]{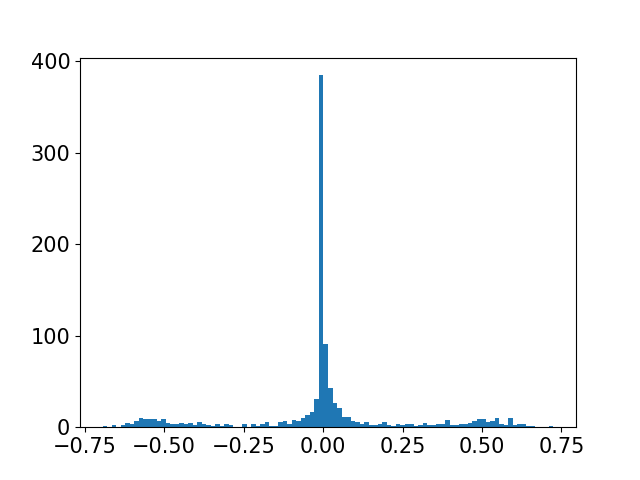}
    \caption{${\bm{r}_{2}^H\circ\bm{r}_{3}^H} - {\bm{r}_{2}^T\circ\bm{r}_{3}^T}$}
    \label{fig:embedding:f}
  \end{subfigure}
  \begin{subfigure}{.23\textwidth}
    \centering
    \includegraphics[width=1.1\linewidth]{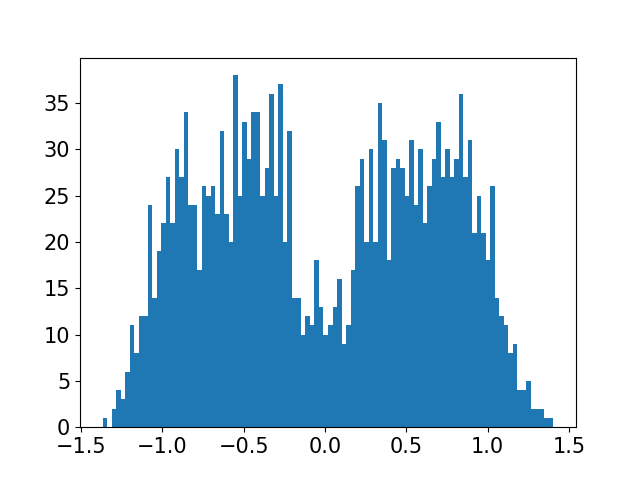}
    \caption{$r_4$}
    \label{fig:embedding:g}
  \end{subfigure}
  \begin{subfigure}{.23\textwidth}
    \centering
    \includegraphics[width=1.1\linewidth]{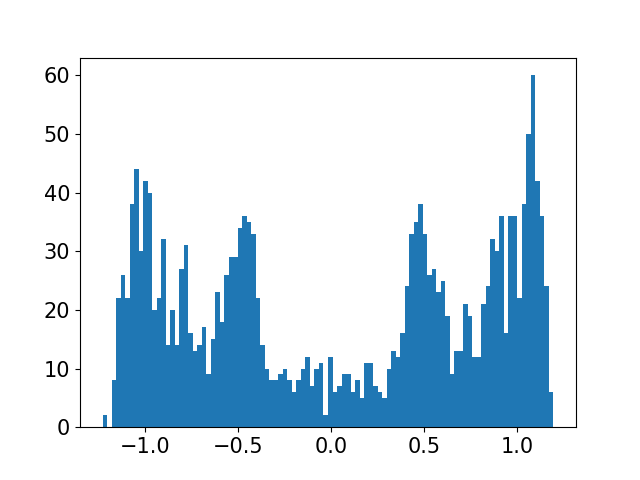}
    \caption{$r_5$}
    \label{fig:embedding:h}
   \end{subfigure}
   
    \begin{subfigure}{.23\textwidth}
    \centering
    \includegraphics[width=1.1\linewidth]{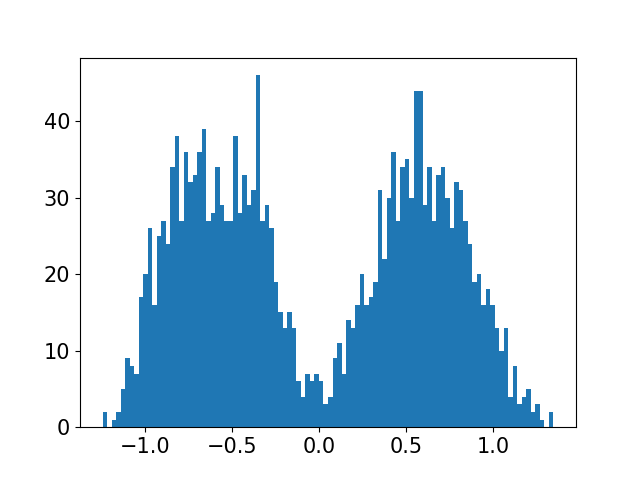}
    \caption{$r_6$}
    \label{fig:embedding:i}
  \end{subfigure}
  \hspace{75mm}
    \begin{subfigure}{.23\textwidth}
    \centering
    \includegraphics[width=1.1\linewidth]{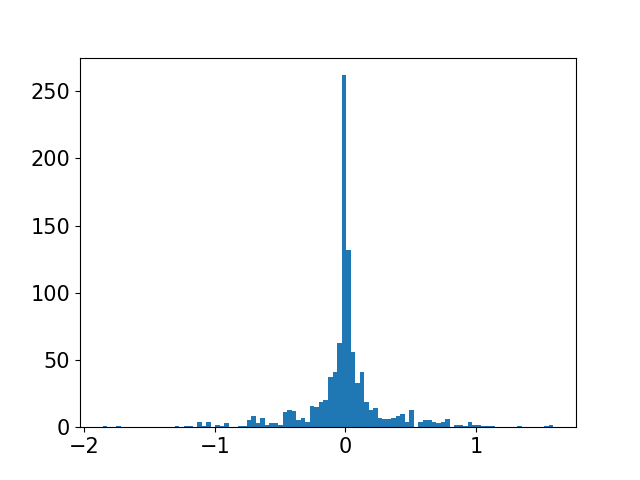}
    \caption{\scriptsize$\bm{r}_4^H\circ\bm{r}_5^H\circ\bm{r}_6^T - \bm{r}_4^T\circ\bm{r}_5^T\circ\bm{r}_6^H $}
    \label{fig:embedding:j}
  \end{subfigure}
\caption{Histograms of relation embeddings for different relation patterns. $r_1$ is relation $spouse$.
$r_2$ is relation $/broadcast/tv\_station/owner$. $r_3$ is relation $/broadcast/tv\_station\_owner/tv\_stations$.
$r_4$ is relation $/location/administrative\_division/capital/location/administrative\_division\-\_capital\_relationship/capital$.
$r_5$ is relation $/location/hud\_county\_place/place$. $r_6$ is relation $base/areas/schema/administrative\_area/capital$.}
\label{fig:symm}
\end{figure*}

\subsection{Model analysis}

\subsection*{Analysis on complex relations} 

We analyze the performances of PairRE for complex relations.
The results of PairRE on different relation categories on FB15k and ogbl-wikikg2 are summarized into Table \ref{table:FB15k&&wikikg}.
We can see PairRE performs quite well on N-to-N and N-to-1 relations. It has a significant lead over baselines. We also notice that performance of 1-to-N relations on ogbl-wikikg2 dataset is not as strong as the other relation categories. One of the reasons is that only 2.2\% of test triples belong to the 1-to-N relation category.

In order to further test the performance of paired relation vectors, we change the relation vector in RotatE to paired vectors.
In the modified RotatE model, both head and tail entities are rotated with different angles based on the paired relation vectors.  This model can also be seen as complex value based PairRE. We name this model as RotatE+PairRelation. The experiment results are shown in Figure ~\ref{fig:r}.
With the same embedding dimension (50 in the experiments),  RotatE+PairRelation improves performance of RotatE with 20.8\%, 27.5\%, 14.4\% and 39.1\% on 1-to-1, 1-to-N, N-to-1 and N-to-N relation categories respectively. These significant improvements prove the superior ability of  paired relation vectors to handle complex relations.

\subsection*{Analysis on relation patterns}

To further verify the learned relation patterns, we visualize some examples.
Histograms of the learned relation embeddings are shown in Figure ~\ref{fig:symm} .

\textbf{Symmetry/AntiSymmetry}. Figure ~\ref{fig:embedding:a} shows a symmetry relation $spouse$ from DB100k.
The embedding dimension is 500.
For PairRE, symmetry relation pattern can be encoded when embedding $\bm{r}$ satisfies ${\bm{r}^H}^2 = {\bm{r}^T}^2$.
Figure ~\ref{fig:embedding:b} shows most of the paired elements in $\bm{r}^H$ and $\bm{r}^T$ have the same absolute value.
Figure ~\ref{fig:embedding:c} shows a antisymmetry relation $tv\_station\_owner$, where most of the paired elements do not have the same absolute value as shown in Figure ~\ref{fig:embedding:d}.

\textbf{Inverse}. Figure ~\ref{fig:embedding:c} and Figure ~\ref{fig:embedding:e} show an example of inverse relations from FB15k.
As the histogram in Figure ~\ref{fig:embedding:f} shows these two inverse relations $tv\_station\_owner$ ($r_2$) and $tv\_station\_owner\_tv\_stations$ ($r_3$) close to satisfy $\bm{r}_3^H \circ \bm{r}_2^H = \bm{r}_3^T \circ \bm{r}_2^T$.

\textbf{Composition}. Figures ~\ref{fig:embedding:g},~\ref{fig:embedding:h}, ~\ref{fig:embedding:i} show an example of composition relation pattern from FB15k, where the third relation $r_6$ can be seen as the composition of the first relation $r_4$ and the second relation $r_5$.
As Figure ~\ref{fig:embedding:j} shows these three relations close to
satisfy $\bm{r}_4^H\circ\bm{r}_5^H\circ\bm{r}_6^T - \bm{r}_4^T\circ\bm{r}_5^T\circ\bm{r}_6^H $.

\section{Conclusion}
To better handle complex relations and tackle more relation patterns, we proposed PairRE, which represents each relation with paired vectors.
With a slight increase in complexity, PairRE can solve the aforementioned two problems efficiently.
Beyond the symmetry/antisymmetry, inverse and composition relations, PairRE can further encode subrelation with simple constraint on relation representations.
On large scale benchmark ogbl-wikikg2 an ogbl-biokg, PairRE outperforms all the state-of-the-art baselines. 
Experiments on other well designed benchmarks also demonstrate the effectiveness of the focused key abilities.

\bibliographystyle{acl_natbib}
\bibliography{anthology,acl2021}

\end{document}